\documentclass[10pt,twocolumn,letterpaper]{article}
\usepackage{amsmath}
\usepackage{amssymb}
\usepackage{amsthm}
\usepackage{bm}
\usepackage{cvpr}
\usepackage{dsfont}
\usepackage{epsfig}
\usepackage{float}
\usepackage{graphics}
\usepackage{times}
\usepackage[pagebackref=true,breaklinks=true,letterpaper=true,colorlinks,bookmarks=false]{hyperref}

\newcommand{\commentout}[1]{}

\newtheorem{proposition}{Proposition}

\renewcommand{\qed}{\rule{5pt}{5pt}}

\newcommand{\bc}{{\bf c}}

\newcommand{\bv}{{\bf v}}

\newcommand{\bx}{{\bf x}}

\newcommand{\bell}{{\bm \ell}}

\newcommand{\cL}{\mathcal{L}}

\newcommand{\eps}{\varepsilon}

\newcommand{\realset}{\mathbb{R}}

\newcommand{\abs}[1]{\left|#1\right|}

\newcommand{\I}[1]{\mathds{1} \! \left\{#1\right\}}

\newcommand{\normw}[2]{\left\|#1\right\|_{#2}}

\newcommand{\set}[1]{\left\{#1\right\}}
\newcommand{\sgn}{\mathrm{sgn}}

\newcommand{\transpose}{^\mathsf{\scriptscriptstyle T}}

\cvprfinalcopy
\def\cvprPaperID{3}

\ifcvprfinal\fi

\begin{document}

\title{Online Semi-Supervised Perception: Real-Time Learning without Explicit Feedback}

\author{Branislav Kveton \\
Intel Labs \\
Santa Clara, CA \\
{\tt\small branislav.kveton@intel.com} \\
\\
Matthai Philipose \\
Intel Labs \\
Seattle, WA \\
{\tt\small matthai.philipose@intel.com}  \and
Michal Valko \\
Department of Computer Science \\
University of Pittsburgh \\
{\tt\small michal@cs.pitt.edu} \\
\\
Ling Huang \\
Intel Labs \\
Berkeley, CA \\
{\tt\small ling.huang@intel.com}}

\maketitle
\thispagestyle{empty}

\begin{abstract}
This paper proposes an algorithm for real-time learning without explicit feedback. The algorithm combines the ideas of semi-supervised learning on graphs and online learning. In particular, it iteratively builds a graphical representation of its world and updates it with observed examples. Labeled examples constitute the initial bias of the algorithm and are provided offline, and a stream of unlabeled examples is collected online to update this bias. We motivate the algorithm, discuss how to implement it efficiently, prove a regret bound on the quality of its solutions, and apply it to the problem of real-time face recognition. Our recognizer runs in real time, and achieves superior precision and recall on 3 challenging video datasets.
\end{abstract}

\section{Introduction}
\label{sec:introduction}

Semi-supervised learning is a field of machine learning that studies learning from both labeled and unlabeled examples. This learning paradigm is extremely useful for solving real-world problems, where data are often abundant but the resources to label them are limited. Therefore, it is not surprising that many semi-supervised learning algorithms have been proposed recently \cite{zhu08semisupervised}. One of the most popular methods is to compute the harmonic function solution on the data adjacency graph \cite{zhu03semisupervised}, and use it to infer labels of unlabeled examples.

This paper investigates an online learning formulation of this problem. In particular, learning is viewed as a repeating game against a potentially adversarial nature. At each step $t$ of this game, we observe an example $\bx_t$ and then predict its label $\hat{y}_t$. The challenge of the game is that we rarely observe the true label $y_t$. Thus, if we want to adapt to changes in the environment, we have to rely on indirect forms of feedback, such as the manifold of data. In this work, we track the data adjacency graph and infer labels of unlabeled data based on the harmonic function solution on this graph. Our approach has several favorable properties. First, it retains the random walk interpretation of the harmonic function solution. Thus, it generalizes beyond binary classification and is also robust to outliers. Second, the quality of our solutions is bounded. Finally, we track the manifold of data and therefore, we can adapt to changing data over time.

Our paradigm is suitable for designing adaptive machine learning algorithms. Labeled examples constitute the initial bias and are provided offline, and a stream of unlabeled data is gathered online to update the bias. Despite the impact that this paradigm may have on building learning algorithms for real-world problems, little work has been done on this topic \cite{babenko09visual,goldberg08online,grabner08semisupervised}.

To illustrate and validate our ideas, we focus on the problem of adaptive face recognition. Our objective is to build a high-quality face recognizer from streams of unlabeled data and a small set of labeled faces. Although we achieve superior results, note that our main objective is not a comparison to other face recognizers. Rather, we wanted to demonstrate the value of unlabeled data in this domain.

The paper has the following structure. In Section \ref{sec:SSL face recognition}, we review the harmonic function solution \cite{zhu03semisupervised} and discuss how to use it for face recognition. In Section \ref{sec:online HFS}, we \mbox{introduce our} learning algorithm, discuss its efficient implementation, and analyze it. In Section \ref{sec:experiments}, our method is empirically evaluated on three datasets. A comparison to the existing work is done in Section \ref{sec:existing work}.

The following notation is used in the paper. The symbols $\bx_i$ and $y_i$ denote the $i$-th example and its label, respectively. The examples $\bx_i$ are divided into labeled and unlabeled sets, $l$ and $u$, and labels $y_i \! \in \! \set{-1, 1}$ are observed for the labeled data only.\footnote{For simplicity of exposition, we assume that the label $y_i$ is binary. Our ideas straightforwardly generalize to multi-class classification \cite{balcan05application}.} The cardinality of the labeled and unlabeled sets is $n_l = \abs{l}$ and $n_u = \abs{u}$, respectively, and the total number of training examples is $n = n_l + n_u$.

\section{Semi-supervised face recognition}
\label{sec:SSL face recognition}

This section has two parts. First, we review the harmonic function solution of Zhu \etal~\cite{zhu03semisupervised} and show how to regularize it to control the extrapolation to unlabeled data. Second, we discuss how semi-supervised learning on a graph can be applied to face recognition.

\subsection{Harmonic function solution}
\label{sec:HFS}

A standard approach to learning on partially labeled data is to minimize the quadratic objective function \cite{zhu03semisupervised}:
\begin{align}
  \min_{\bell \in \realset^n} \bell\transpose L \bell
  \quad \textrm{s.t.} \
  \ell_i = y_i \textrm{ for all } i \in l; \label{eq:HFS}
\end{align}
where $\bell$ denotes the vector of predictions, $L = D - W$ is the Laplacian of the data adjacency graph, which is represented by a matrix $W$ of pairwise similarities $w_{ij}$, and $D$ is a diagonal matrix whose entries are given by $d_i = \sum_j w_{ij}$. This problem has a closed-form solution:
\begin{align}
  \bell_u = (D_{uu} - W_{uu})^{-1} W_{ul} \bell_l,
  \label{eq:closed-form HFS}
\end{align}
which satisfies the \emph{harmonic property} $\ell_i = \frac{1}{d_i} \sum_{j \sim i} w_{ij} \ell_j$, and therefore is commonly known as the \emph{harmonic function solution}. Since the solution can be also computed as:
\begin{align}
  \bell_u = (I - P_{uu})^{-1} P_{ul} \bell_l,
  \label{eq:random walk HFS}
\end{align}
it can be viewed as a product of a random walk on the graph $W$ with the transition matrix $P = D^{-1} W$. The probability of moving between two arbitrary vertices $i$ and $j$ is $w_{ij} / d_i$, and the walk terminates when the reached vertex is labeled. Each element of the solution is given by:
\begin{align}
  \ell_i
  \ = & \ \ (I - P_{uu})_{iu}^{-1} P_{ul} \bell_l \nonumber \\
  \ = & \ \
  \underbrace{\sum_{j: y_j = 1} (I - P_{uu})_{iu}^{-1}
  P_{uj}}_{p_i^1} -
  \underbrace{\sum_{j: y_j = -1} (I - P_{uu})_{iu}^{-1}
  P_{uj}}_{p_i^{-1}} \nonumber \\
  \ = & \ \ p_i^1 - p_i^{-1},
  \label{eq:probability HFS}
\end{align}
where $p_i^1$ and $p_i^{-1}$ are probabilities by which the walk starting from the vertex $i$ ends at vertices with labels $1$ and $-1$, respectively. Therefore, when $\ell_i$ is rewritten as $\abs{\ell_i} \sgn(\ell_i)$, $\abs{\ell_i}$ can be interpreted as a \emph{confidence} of assigning the label $\sgn(\ell_i)$ to the vertex $i$. The maximum value of $\abs{\ell_i}$ is 1, and it is achieved when either $p_i^1 = 1$ or $p_i^{-1} = 1$. The closer the confidence $\abs{\ell_i}$ to 0, the closer the probabilities $p_i^1$ and $p_i^{-1}$ to 0.5, and the more \emph{uncertain} the label $\sgn(\ell_i)$.

To control the amount of extrapolation to unlabeled data, we regularize the Laplacian $L$ as $L + \gamma_g I$, where $\gamma_g$ is a non-negative scalar and $I$ is the identity matrix. Similarly to the problem (\ref{eq:HFS}), the corresponding harmonic function solution:
\begin{align}
  \min_{\bell \in \realset^n} \bell\transpose (L + \gamma_g I) \bell
  \quad \textrm{s.t.} \
  \ell_i = y_i \textrm{ for all } i \in l \label{eq:reg HFS}
\end{align}
can be computed in a closed form:
\begin{align}
  \bell_u = (L_{uu} + \gamma_g I)^{-1} W_{ul} \bell_l.
  \label{eq:closed-form reg HFS}
\end{align}
It can be also interpreted as a random walk on the graph $W$ with an extra sink. At each step, this walk may terminate at the sink with probability $\gamma_g / (d_i + \gamma_g)$. Thus, the parameter $\gamma_g$ essentially controls how the confidence of labeling unlabeled examples drops with the number of hops from labeled examples.

When $\gamma_g = 0$, the regularized solution (\ref{eq:reg HFS}) turns into the ordinary harmonic function solution (\ref{eq:HFS}). When $\gamma_g \! = \! \infty$, the confidence of labeling unlabeled vertices \mbox{decreases to zero.} Finally, note that our regularization corresponds to increasing all eigenvalues of the Laplacian $L$ by $\gamma_g$. In Section \ref{sec:theoretical analysis}, we use this property to bound the regret of our solutions.

\subsection{Face recognition}
\label{sec:face recognition}

Face recognition can be formulated as a semi-supervised learning problem on the data adjacency graph (Section \ref{sec:HFS}). The vertices of the graph are faces, the weights on its edges reflect the similarity of the faces, and the harmonic function solution on the graph yields the identity of the faces.

In our paper, the similarity of faces is computed as $w_{ij} = \exp\!\left[- \frac{d^2(\bx_i, \bx_j)}{2 \sigma^2}\right]$, where $\sigma$ is a heat \mbox{parameter and $d(\bx_i, \bx_j)$} is the distance of the faces in the feature space. The distance is given by:
\begin{align}
  d(\bx_i, \bx_j) = \min\left\{\!\!
  \begin{array}{l}
    \normw{\bx_i - \bx_j}{2, \psi}, \\
    \normw{(\bx_i - \bar{\bx}_i) - (\bx_j - \bar{\bx}_j)}{2, \psi}, \\
    \normw{\bx_i / \bar{\bx}_i - \bx_j / \bar{\bx}_j}{2, \psi}
  \end{array}
  \!\!\right\},
  \label{eq:face distance}
\end{align}
where $\bx_i$ and $\bx_j$ are pixel intensities in $96 \times 96$ face images, $\bar{\bx}_i$ and $\bar{\bx}_j$ are mean values of the intensities, and $\normw{\cdot}{2, \psi}$ is a weighted $\cL_2$-norm that gives higher weights to pixels in the centers of the images. At a high level, the function $d(\bx_i, \bx_j)$ measures the distance of two raw images, and corrects it for additive and multiplicative light. Undoubtedly, this distance function is very simple. Yet, it yields extremely good results in all of our experiments (Section \ref{sec:experiments}) and is likely to perform well on popular vision datasets \cite{pinto09howfar}.

The heat parameter is set as $\sigma = 0.025$. For this setting, the similarity of any two different faces from the SZSL subset of the MPLab GENKI database \cite{genki-szsl} is at most $10^{-6}$. To make the graph $W$ sparse, we turn it into an $\eps$-neighborhood graph. In particular, we set $w_{ij}$ to 0 whenever $w_{ij} < \eps$. As a result of this transformation, some faces may be completely disconnected from the rest of the graph. Note that the regularized harmonic function solution on these faces is 0. Thus, there is no preference for their labels and may treat the faces as \emph{outliers}. In addition, we may also refrain from predicting their labels.

In the experimental section, we vary $\eps$ and study its effect on the precision and recall of our learner. For simplicity, we set the regularization parameter $\gamma_g$ as $\gamma_g = 10 \eps$. Intuitively, the more we extrapolate to unlabeled examples, the lower is the penalty $\gamma_g$ for this extrapolation.

\section{Online harmonic function solution}
\label{sec:online HFS}

The regularized harmonic function solution (Section \ref{sec:HFS}) is an offline learning algorithm. A trivial way of making the algorithm online is to maintain the complete data adjacency graph up to each time step $t$ and then use it to infer the label of the most recent example $\bx_t$. This solution is not practical because its time complexity grows with time $t$ and is $O(t^3)$.

\subsection{Quantization}
\label{sec:quantization}

To address the problem, we employ \emph{data quantization} \cite{gray98quantization} and maintain a compact representation of the complete data adjacency graph. Before we discuss details of our approach, we show that if the complete graph $\tilde{W}_t$ up to time $t$ involves identical vertices, the harmonic function solution on $\tilde{W}_t$ can be computed compactly on a smaller graph. Since $n_u \gg n_l$, we mainly focus on the quantization of unlabeled examples.

\begin{proposition}
\label{prop:compact HFS} Let $W$ be a graph, which is derived from the graph $\tilde{W}$ by deleting all but a single instance of all identical vertices. Moreover, let:
\begin{align*}
  \hat{W} = V W V
\end{align*}
be a matrix, whose rows and columns are multiplied by the corresponding number of identical vertices $\bv$ in $\tilde{W}$, and $V$ be a diagonal matrix such that $V_{ii} = v_i$. Then the harmonic function solution (\ref{eq:reg HFS}) on $\tilde{W}$ can be computed compactly as:
\begin{align*}
  \hat{\bell}_u =
  (\hat{L}_{uu} + \gamma_g V)^{-1} \hat{W}_{ul} \bell_l,
\end{align*}
where $\hat{L}$ is the Laplacian of $\hat{W}$.
\end{proposition}
\noindent {\bf Proof:} Our proof is based on the electric circuit interpretation of a random walk \cite{zhu03semisupervised}. More specifically, we show that $\tilde{W}$ and $\hat{W}$ represent identical electric circuits and therefore, their harmonic function solutions are the same.

In the electric circuit formulation of $\tilde{W}$, the edges of the graph are resistors with the conductance $\tilde{w}_{ij}$. If two vertices $i$ and $j$ are identical, then $\tilde{w}_{ij} = 1$ and they can be viewed as resistors in parallel. The total capacitance of two resistors in parallel is equal to the sum of their capacitances. Therefore, the two resistors can be replaced by a single resistor with the capacitance of the sum. A repetitive application of this rule yields $\hat{W} = V W V$.

In Section \ref{sec:HFS}, we showed that the regularized harmonic function solution can be interpreted as having an extra sink in a graph. Therefore, when two vertices $i$ and $j$ are merged, we also need to sum up their sinks. A repetitive application of this rule yields the term $\gamma_g V$ in our closed-form solution. \qed

\begin{figure}[t]
  \centering
  \rule{\linewidth}{0.01in}
  \small{
  \begin{tabbing}
    \hspace{0.1in} \= \hspace{0.1in} \= \hspace{0.1in} \= \hspace{0.1in} \= \kill
    {\bf Inputs:} \\
    \> an unlabeled example $\bx_t$ \\
    \> a set of representative vertices $C_{t - 1}$ \\
    \> vertex multiplicities $\bv_{t - 1}$\\
    \\
    {\bf Algorithm:} \\
    \> if $(\abs{C_{t - 1}} = n_g + 1)$ \\
    \>\> $R = 2 R$ \\
    \>\> greedily repartition $C_{t - 1}$ into $C_t$ such that: \\
    \>\>\> no two vertices in $C_t$ are closer than $R$ \\
    \>\>\> for any $\bc_i \in C_{t - 1}$ exists $\bc_j \in C_t$ such that $d(\bc_i, \bc_j) < R$ \\
    \>\> update $\bv_t$ to reflect the new partitioning \\
    \> else \\
    \>\> $C_t = C_{t - 1}$ \\
    \>\> $\bv_t = \bv_{t - 1}$ \\
    \> if $\bx_t$ is closer than $R$ to any $\bc_i \in C_t$ \\
    \>\> $\bv_t(i) = \bv_t(i) + 1$ \\
    \> else \\
    \>\> $\bv_t(\abs{C_t} + 1) = 1$ \\
    \>\> add $\bx_t$ to the position $(\abs{C_t} + 1)$ in $C_t$ \\
    \> build a similarity matrix $W_t$ over the vertices $C_t$ \\
    \> build a matrix $V_t$ whose diagonal elements are $\bv_t$ \\
    \> $\hat{W}_t = V_t W_t V_t$ \\
    \> compute the Laplacian $\hat{L}$ of the graph $\hat{W}_t$ \\
    \> infer labels on the graph: \\
    \>\> $\displaystyle \hat{\bell}[t]  =
    \arg\min_\bell \bell\transpose (\hat{L} + \gamma_g V_t) \bell$ \\
    \>\> $\textrm{s.t.} \ \ell_i = y_i$
    for all labeled examples up to time $t$ \\
    \> make a prediction $\hat{y}_t = \sgn(\hat{\ell}_t[t])$ \\
    \\
    {\bf Outputs:} \\
    \> a prediction $\hat{y}_t$ \\
    \> a set of representative vertices $C_t$ \\
    \> vertex multiplicities $\bv_t$
  \end{tabbing}
  }
  \vspace{-0.07in}
  \rule{\linewidth}{0.01in}
  \caption{Computation of the online harmonic function solution at time $t$. The main parameter of the method is the maximum number of representative vertices $n_g$.}
  \label{fig:online quantized HFS}
\end{figure}

\bigskip Proposition \ref{prop:compact HFS} implies that the harmonic function solution on a graph with at most $n_g$ distinct vertices can be computed in $O(n_g^3)$ time steps. The time complexity of this computation is independent of $t$. Therefore, a data adjacency graph $W_t$ with a fixed number of representative vertices $n_g$ seems to be a perfect compact representation of $\tilde{W}_t$. The graph can be updated on-the-fly and incrementally using the doubling algorithm of Charikar \etal~\cite{charikar97incremental}.

\begin{figure*}[t]
  \centering
  \includegraphics[width=6.8in, bb=0in 3.25in 8.5in 7.75in]{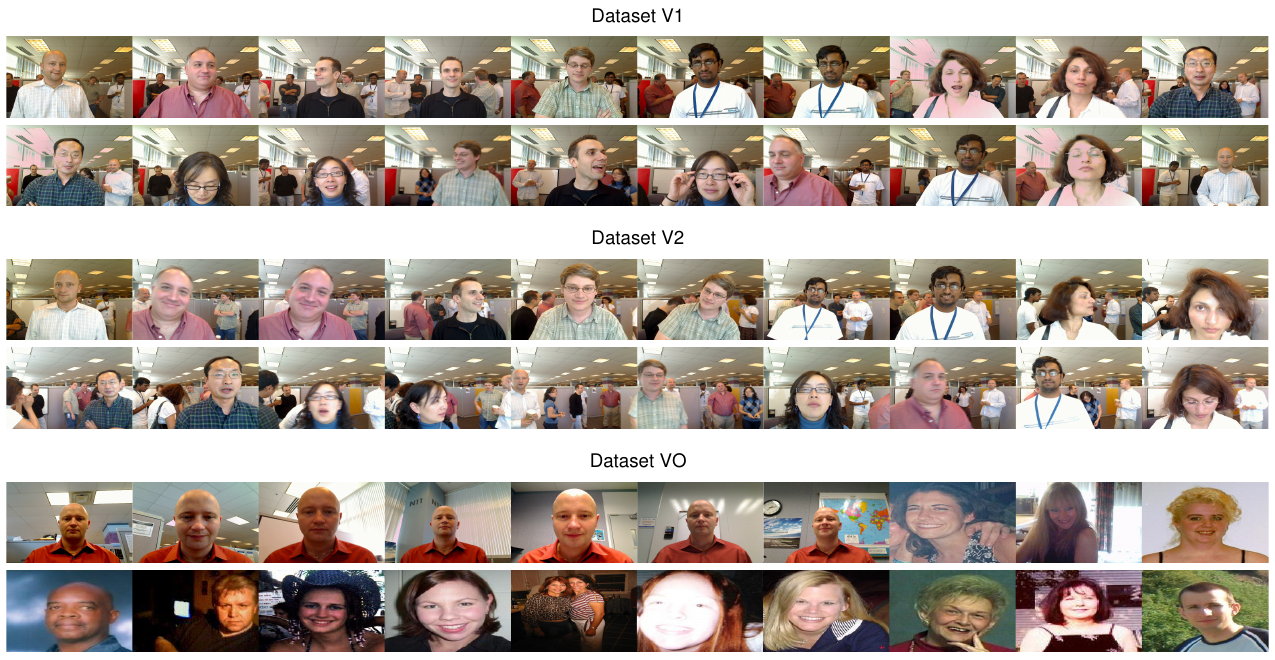}
  \caption{Snapshots from the datasets V1, V2, and VO.}
  \label{fig:videos}
\end{figure*}

The \emph{doubling algorithm} maintains a set of representative vertices $C_t = \set{\bc_1, \bc_2, \dots}$ such that the distance between any two vertices in $C_t$ is at least $R$. When a new \mbox{example $\bx_t$} appears and its distance from any $\bc_i \in C_t$ is less than $R$, the example is merged with $\bc_i$. When the distance of $\bx_t$ from all $\bc_i \in C_t$ is at least $R$, $\bx_t$ is added to the set of representative vertices $C_t$. Finally, when $\abs{C_t} \! > \! n_g$, the scalar $R$ is doubled and $C_t$ is greedily repartitioned such that no two vertices in $C_t$ are closer than $R$.

The advantage of these updates is that they provide guarantees on the quality of our approximation. In particular, at any point in time $t$, the distance of any example $\bx_j$ from its representative vertex $\bc_i$ is at most than $2 R$ \cite{charikar97incremental}.

\subsection{Theoretical analysis}
\label{sec:theoretical analysis}

\begin{figure}
  \centering
  \includegraphics[width=3.2in, bb=1.25in 1in 4.75in 2.85in]{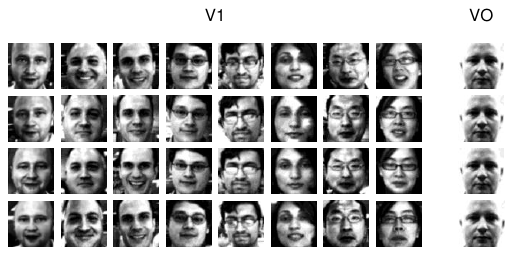}
  \caption{Labeled faces in the datasets V1 and VO.}
  \label{fig:labeled faces}
\end{figure}

The error in the predictions of our learner (Figure \ref{fig:online quantized HFS}) can be decomposed into 3 error terms and bounded. The proofs of the bounds are out of the scope of this work and therefore, we only outline them.

In the rest of the section, we use $\bell^\ast$, $\tilde{\bell}[t]$, and $\hat{\bell}[t]$ to refer to the harmonic function solutions on the full data adjacency graph $W$, its observed portion up to time $t$, and its quantized approximation, respectively; and $\ell_t^\ast$, $\tilde{\ell}_t[t]$, and $\hat{\ell}_t[t]$ refer to the corresponding solution on the vertex $\bx_t$. Our analysis is based on the observation that we solve a regression problem where the goal is to minimize the error $\sum_t (\hat{\ell}_t[t] - y_t)^2$. This error can be rewritten as a sum of 3 terms:
\begin{align}
  \frac{1}{n} \sum_t (\hat{\ell}_t[t] - y_t)^2
  \ \leq & \ \ \frac{9}{2 n} \sum_t (\ell_t^\ast - y_t)^2 +
  \nonumber \\
  & \ \ \frac{9}{2 n} \sum_t (\tilde{\ell}_t[t] - \ell_t^\ast)^2 +
  \nonumber \\
  & \ \ \frac{9}{2 n} \sum_t (\hat{\ell}_t[t] - \tilde{\ell}_t[t])^2,
\end{align}
which represent the errors due to the harmonic function solution, online learning, and data quantization. The first term $\frac{9}{2 n} \sum_t (\ell_t^\ast - y_t)^2$ can be decomposed into the empirical risk on labeled vertices and another error, which decreases at the rate of $O(n_l^{- \frac{1}{2}})$ when $\gamma_g = \Omega(n_l^\frac{3}{2})$ \cite{kveton10semisupervised}. In a similar fashion, the other terms $\frac{9}{2 n} \sum_t (\tilde{\ell}_t[t] - \ell_t^\ast)^2$ and $\frac{9}{2 n} \sum_t (\hat{\ell}_t[t] - \tilde{\ell}_t[t])^2$ can be bounded on the order of $O(n^{- \frac{1}{2}})$ when $\gamma_g = \Omega(n^\frac{1}{4})$. Since $n \gg n_l$, we choose $\gamma_g \! = \! \Omega(n^\frac{1}{4})$ and get the following regret bound:
\begin{align}
  \frac{1}{n} \sum_t (\hat{\ell}_t[t] - y_t)^2 \leq
  \frac{9}{2 n_l} \sum_{i \in l} (\ell_i^\ast - y_i)^2 +
  O(n^{- \frac{1}{2}}).
\end{align}
This bound can be interpreted as follows. When our learner is regularized enough, its per-step regret decreases over time at the rate of $O(n^{- \frac{1}{2}})$.

\begin{figure*}[t]
  \centering
  \includegraphics[width=2in, bb=3in 4.5in 5.5in 6.5in]{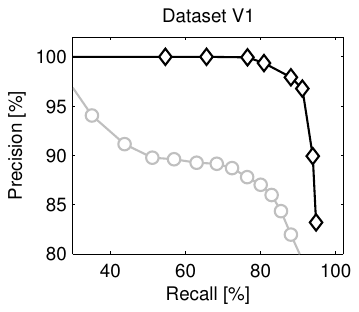}
  \vspace{0.05in}
  \includegraphics[width=2in, bb=3in 4.5in 5.5in 6.5in]{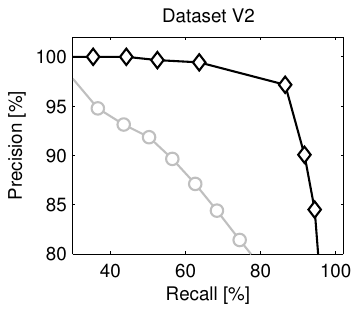}
  \vspace{0.05in}
  \includegraphics[width=2in, bb=3in 4.5in 5.5in 6.5in]{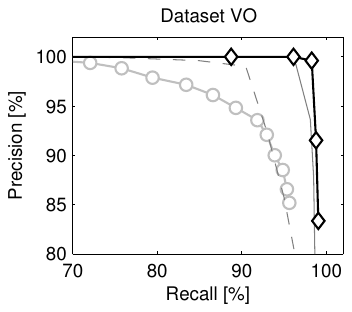}
  \caption{Comparison of 3 face recognizers on the datasets V1, V2, and VO. The recognizers are trained by a NN classifier (gray lines with circles), online semi-supervised boosting (thin gray lines), and our learner (black lines with diamonds). The plots are \mbox{generated by varying} the radius of our $\eps$ neighborhoods (Section \ref{sec:face recognition}). From left to right, the points on the plots correspond to decreasing values of $\eps$. The boosted solutions are learned from 500 weak learners, which uniformly cover the whole dataset VO (solid line), and its first and last quarters (dashed line).}
  \label{fig:results}
\end{figure*}

\subsection{Outliers}
\label{sec:outliers}

Our online learner (Figure \ref{fig:online quantized HFS}) is not very robust to outliers because it always predicts. To make the learner more robust, we alter it in two ways. First, when $\hat{\ell}_t[t] = 0$, which means that the vertex $\bx_t$ is an outlier (Section \ref{sec:face recognition}), we refrain from predicting. Second, when the vertex is an outlier, we ignore it when updating the set of representative vertices. This rule can be viewed as representing all outliers by a single vertex, which has zero impact on the harmonic function solution on $\hat{W}_t$.

\section{Experiments}
\label{sec:experiments}

The experimental section is divided into three parts. The first part evaluates our learner on an 8-way face recognition problem. The learner is also compared to a nearest-neighbor classifier, which is trained offline on labeled examples. The second part demonstrates that our learner can partially adapt to sudden changes in the environment, such as varying light conditions. At the same time, the learner seems to be robust to outliers and outperforms online semi-supervised boosting \cite{grabner08semisupervised}. Finally, we study the tradeoff between the quality of our solutions and the number of representative vertices $n_g$.

In the first two experiments, our online learner \mbox{(Figure \ref{fig:online quantized HFS})} maintains $n_g = 500$ representative vertices, and we vary the radius of our $\eps$ neighborhoods (Section \ref{sec:face recognition}) to get predictors with varying precision and recall. In the last experiment, we fix $\eps$ and vary the number of representative vertices $n_g$.

\subsection{Datasets}
\label{sec:datasets}

To evaluate our algorithm, we collected 3 video datasets: V1, V2, and VO (Figure \ref{fig:videos}). The datasets V1 and V2 involve 8 people who walk in front of two cameras and make funny faces. The cameras are two meters apart and slightly rotated with respect to each other. When the face of a person shows up on the first camera for the first time, we label four frontal faces of the person (Figure \ref{fig:labeled faces}). We do not label any face that was captured by the second camera.

The dataset VO involves a single participant whose faces are captured at different locations, such as a cubicle, the corner with a couch, and a conference room. Only the \mbox{first four} faces of the person are labeled (Figure \ref{fig:labeled faces}), and our main goal is to learn a good face recognizer at all locations. To test the sensitivity of the recognizer to false positives, we appended our dataset by faces from the MPLab GENKI database \cite{genki-szsl}. An ideal recognizer would always recognize our participant but never extrapolate to any of the appended faces.

In all experiments, faces are detected by OpenCV, turned into grayscale, smoothed out, and their histogram is normalized. The quality of face recognition algorithms is measured by their precision and recall. The statistics are computed per frame. If a face recognizer makes multiple different predictions on a single frame, the per-frame prediction is automatically incorrect. This evaluation methodology is suitable for our problem since our videos mostly involve one larger face at a time (Figure \ref{fig:videos}) and none of the faces in the background are detected.

\subsection{Face recognition}
\label{sec:face recognition experiments}

In the first experiment (Figure \ref{fig:results}), we evaluate our learner (Figure \ref{fig:online quantized HFS}) on the datasets V1 and V2. On both datasets, the learner can achieve 95 percent precision at 90 percent recall levels. This operating point corresponds to $\eps = 10^{-8}$. Generally, as $\eps$ decreases, the recall of our learner \mbox{increases and} its precision goes down. Finally, since no face in the dataset V2 is labeled, our results on this dataset are especially good. In fact, we may conclude that our learner is able to bootstrap from labeled data in a different dataset. We elaborate on this idea in Section \ref{sec:sudden changes experiments}.

Figure \ref{fig:results} also compares our online learner to the nearest-neighbor (NN) classifier on labeled faces:
\begin{align}
  \hat{y}_t^\mathrm{nn} = \arg\max_c \sum_{i \in l} \I{y_i = c}
  \I{w_{it} \geq \eps} w_{it},
  \label{eq:NN classifier}
\end{align}
where $\I{\cdot}$ is an indicator function. At the same recall levels, the precision of the learner is typically 10 percent higher than the precision of the NN classifier. This improvement is due to tracking the manifold of data.

\subsection{Sudden changes in the environment}
\label{sec:sudden changes experiments}

In the second experiment (Figure \ref{fig:results}), we demonstrate that our online learner (Figure \ref{fig:online quantized HFS}) adapts to sudden changes in the environment, such as varying light conditions and locations. This experiment is performed on the dataset VO, where our participant changes 3 locations and is viewed from 8 camera positions. The online learner achieves 100 percent precision and 95 percent recall. Since a half of the dataset VO consists of images of other people, we may conclude that our learner is capable of adapting to sudden changes in the environment without extrapolating too far to outliers.

Similarly to Section \ref{sec:face recognition experiments}, our learner performs better than the NN classifier. Moreover, we compare our learner to online semi-supervised boosting \cite{grabner08semisupervised}, which is a state-of-the-art method for online semi-supervised learning. To allow for a fair comparison, we modify the method as follows. \mbox{First, all} weak learners are of the nearest-neighbor form:
\begin{align}
  h_i(\bx_t) = \I{w_{it} \geq \eps},
  \label{eq:NN weak learner}
\end{align}
where $\eps$ is the radius of the neighborhood. Second, the class of outliers is modeled implicitly. Particularly, the goal of the new algorithm is to learn a predictor $H(\bx_t) = \sum_i \alpha_i h_i(\bx_t)$ such that $H(\bx_t) = 0$ for outliers and $H(\bx_t) > 0$ otherwise.

Figure \ref{fig:results} shows that online semi-supervised boosting performs as well as our method when given a good set of weak learners. However, future data are rarely known in advance and when the learners $h_i(\bx_t)$ cover only a part of the dataset VO, the quality of the boosted results degrades significantly (Figure \ref{fig:results}). In comparison, our online learner always adapts its representation of the world. How to incorporate a similar step in online semi-supervised boosting is not obvious.

\subsection{Size of data adjacency graphs}
\label{sec:graph size experiments}

In the last experiment (Figure \ref{fig:tradeoff}), we evaluate the performance of our learner (Figure \ref{fig:online quantized HFS}) when varying the number of representative vertices $n_g$. The radius of the $\eps$ \mbox{neighborhood} (Section \ref{sec:face recognition}) is $\eps \! = \! 10^{-5}$. This setting is pretty conservative. Thus, our learner always achieves 100 percent precision and we measure the quality of its solutions using recall.

Our results reveal two trends. First, the computation time of our learner grows superlinearly with $n_g$. This is expected since the computation of the harmonic function solution on $\hat{W}_t$ (Figure \ref{fig:online quantized HFS}) takes $O(n_g^3)$ time steps. Second, the recall of the learner improves as $n_g$ increases. Note that most of this improvement is a result of tracking the first 200 representative vertices.

\begin{figure}
  \centering
  \includegraphics[width=3.2in, bb=2.25in 4.5in 6.25in 6.5in]{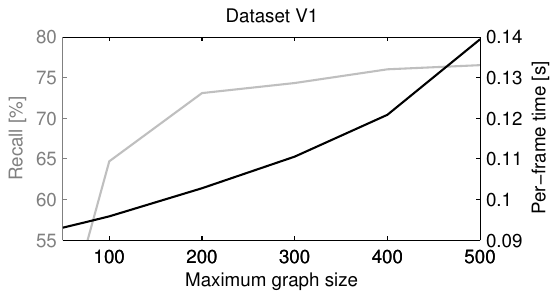}
  \vspace{0.05in}
  \caption{Recall (gray line) and per-frame computation time (black line) of our learner as a function of the maximum graph size $n_g$.}
  \label{fig:tradeoff}
\end{figure}

Finally, note that our face recognizer can run in real time. In particular, even when $n_g = 500$, the recognizer processes about 7 frames per second on average.

\section{Existing work}
\label{sec:existing work}

In this section, we compare our work to the existing work on online semi-supervised learning and face recognition.

\subsection{Online semi-supervised learning}
\label{sec:existing OSSL}

Online learning from partially labeled data should be of a great interest to both machine learning and computer vision communities. Unfortunately, very little work has been done on this topic \cite{babenko09visual,goldberg08online,grabner08semisupervised}. Online semi-supervised boosting and online manifold regularization of SVMs are two notable examples of recently proposed algorithms.

Online manifold regularization of SVMs \cite{goldberg08online} is an online learning algorithm for manifold regularization of SVMs \cite{belkin06manifold}. The algorithm learns max-margin classifiers, which are regularized by the data adjacency graph. This graph serves the same purpose as the graph maintained by our online learner (Section \ref{sec:online HFS}). Therefore, when online manifold regularization is parameterized properly, it may produce the same result as our method. Despite this similarity, our solution has several advantages. First, its parameter space is smaller because we do not learn an additional decision boundary over the graph. Second, the quality of our solution is bounded (Section \ref{sec:theoretical analysis}). Finally, the solution generalizes to multi-class classification \cite{balcan05application} and is robust to outliers.

Online semi-supervised boosting \cite{grabner08semisupervised} is an online version of boosting, where unlabeled examples are labeled greedily using the data adjacency graph. The method learns a binary classifier, where one of the classes tracks the object of interest and the other one models everything else. This classifier is a linear combination of weak learners, which are specified in advance. A good set of the learners is typically unknown in advance and thus, this is a major weakness of the method. In contrast, our solution tracks the manifold of data, and can discover and adapt to non-linear patterns in real time.

\subsection{Face recognition}
\label{sec:existing face recognition}

The problem of face recognition has been studied extensively by the computer vision community \cite{zhao03face}. Most of this research focused on finding better face recognition features. These features can be combined with the temporal model of the environment \cite{arandjelovic09methodology,lee03videobased,zhou03probabilistic} and used for face recognition on videos.

In comparison to this work, we use neither sophisticated features nor temporal models. Our model of human faces is a data adjacency graph, which is built over time by tracking the manifold of data in a sequence of videos. The advantage of our approach is that it learns automatically with very little human feedback. On the other hand, since we do not model the environment, we need to be careful when generalizing to new data points. Finally, note that our data \mbox{adjacency graph} can be defined over more complex features than the ones in Section \ref{sec:face recognition}. Therefore, our method is essentially orthogonal to finding a better set of face recognition features.

\section{Conclusions}
\label{sec:conclusions}

In this work, we study a novel algorithm for online semi-supervised learning. The method iteratively builds a graphical representation of the world and updates it using a stream of unlabeled examples. This framework is extremely useful for designing adaptive learning algorithms and we illustrate it by building a highly accurate face recognizer from simple features.

In our future work, we want to apply our online learner to other domains, where streams of unlabeled data are handily available, such as object recognition. In addition, we would like to enhance our adaptive face recognizer with more complex features, such as Haar-like features. Finally, one of our main future goals is to develop online graph-based learning algorithms, which are significantly faster than the harmonic function solution.

\bibliographystyle{ieee}
\bibliography{References}

\begin{thebibliography}{10}\itemsep=-1pt

\bibitem{arandjelovic09methodology}
O.~Arandjelovic and R.~Cipolla.
\newblock A methodology for rapid illumination-invariant face recognition using
  image processing filters.
\newblock {\em Computer Vision and Image Understanding}, 113(2):159--171, 2009.

\bibitem{babenko09visual}
B.~Babenko, M.-H. Yang, and S.~Belongie.
\newblock Visual tracking with online multiple instance learning.
\newblock In {\em Proceedings of the IEEE Computer Society Conference on
  Computer Vision and Pattern Recognition}, 2009.

\bibitem{balcan05application}
M.-F. Balcan, A.~Blum, P.~P. Choi, J.~Lafferty, B.~Pantano, M.~R. Rwebangira,
  and X.~Zhu.
\newblock Person identification in webcam images: An application of
  semi-supervised learning.
\newblock In {\em ICML 2005 Workshop on Learning with Partially Classified
  Training Data}, 2005.

\bibitem{belkin06manifold}
M.~Belkin, P.~Niyogi, and V.~Sindhwani.
\newblock Manifold regularization: A geometric framework for learning from
  labeled and unlabeled examples.
\newblock {\em Journal of Machine Learning Research}, 7:2399--2434, 2006.

\bibitem{charikar97incremental}
M.~Charikar, C.~Chekuri, T.~Feder, and R.~Motwani.
\newblock Incremental clustering and dynamic information retrieval.
\newblock In {\em Proceedings of the 29th Annual ACM Symposium on Theory of
  Computing}, pages 626--635, 1997.

\bibitem{goldberg08online}
A.~Goldberg, M.~Li, and X.~Zhu.
\newblock Online manifold regularization: A new learning setting and empirical
  study.
\newblock In {\em Proceeding of European Conference on Machine Learning and
  Principles and Practice of Knowledge Discovery in Databases}, 2008.

\bibitem{grabner08semisupervised}
H.~Grabner, C.~Leistner, and H.~Bischof.
\newblock Semi-supervised on-line boosting for robust tracking.
\newblock In {\em Proceedings of the 10th European Conference on Computer
  Vision}, pages 234--247, 2008.

\bibitem{gray98quantization}
R.~Gray and D.~Neuhoff.
\newblock Quantization.
\newblock {\em IEEE Transactions on Information Theory}, 44(6):2325--2383,
  1998.

\bibitem{kveton10semisupervised}
B.~Kveton, M.~Valko, A.~Rahimi, and L.~Huang.
\newblock Semi-supervised learning with max-margin graph cuts.
\newblock In {\em Proceedings of the 13th International Conference on
  Artificial Intelligence and Statistics}, pages 421--428, 2010.

\bibitem{lee03videobased}
K.-C. Lee, J.~Ho, M.-H. Yang, and D.~Kriegman.
\newblock Video-based face recognition using probabilistic appearance
  manifolds.
\newblock In {\em Proceedings of the IEEE Computer Society Conference on
  Computer Vision and Pattern Recognition}, pages 313--320, 2003.

\bibitem{pinto09howfar}
N.~Pinto, J.~DiCarlo, and D.~Cox.
\newblock How far can you get with a modern face recognition test set using
  only simple features?
\newblock In {\em Proceedings of the IEEE Computer Society Conference on
  Computer Vision and Pattern Recognition}, 2009.

\bibitem{genki-szsl}
\url{http://mplab.ucsd.edu}.
\newblock {MPLab GENKI Database}.

\bibitem{zhao03face}
W.-Y. Zhao, R.~Chellappa, P.~Phillips, and A.~Rosenfeld.
\newblock Face recognition: A literature survey.
\newblock {\em ACM Computing Surveys}, 35(4):399--458, 2003.

\bibitem{zhou03probabilistic}
S.~Zhou, V.~Kruger, and R.~Chellappa.
\newblock Probabilistic recognition of human faces from video.
\newblock {\em Computer Vision and Image Understanding}, 91(1-2):214--245,
  2003.

\bibitem{zhu08semisupervised}
X.~Zhu.
\newblock Semi-supervised learning literature survey.
\newblock Technical Report 1530, University of Wisconsin-Madison, 2008.

\bibitem{zhu03semisupervised}
X.~Zhu, Z.~Ghahramani, and J.~Lafferty.
\newblock Semi-supervised learning using gaussian fields and harmonic
  functions.
\newblock In {\em Proceedings of the 20th International Conference on Machine
  Learning}, pages 912--919, 2003.

\end{thebibliography}

\end{document}